\documentclass{article}
\usepackage{preprint_style,times}
\finalcopy % arXiv preprint: non-anonymous author; no venue header (clean preprint)
\usepackage{microtype}

\usepackage{amsmath,amssymb,amsthm}
\usepackage{hyperref}
\usepackage{url}
\usepackage{booktabs}
\usepackage{multirow}
\usepackage{graphicx}
\usepackage{xcolor}
\usepackage{enumitem}
\graphicspath{{figures/}{./}}

\newcommand{\gamg}{\gamma_g}
\newcommand{\DeltaU}{\|\Delta U\|_F}
\newcommand{\Unew}{U_{\rm new}}
\newcommand{\Uold}{U_{\rm old}}

\newcommand{\RR}{\mathbb{R}}
\newcommand{\EE}{\mathbb{E}}

\newtheorem{theorem}{Theorem}
\newtheorem{proposition}[theorem]{Proposition}

\newtheorem{remark}[theorem]{Remark}

\title{No Subspace to Track:\\
Non-Identifiability and Optimizer State in Low-Rank Training}

\author{Noel Thomas \\
Mohamed bin Zayed University of Artificial Intelligence \\
\texttt{noel.thomas@mbzuai.ac.ae}}

\begin{document}

\maketitle

% Abstract v2 (final push 2026-07-03). Claim -> evidence -> impact. No em dashes.
% Numbers: EC-20 (same-step), EC-22 (N^{-1/4}), EC-24 (84%), EC-23 (flat beta3 row),
% EC-25 (transport 2x2, 3 seeds), EC-04/EC-26 (beta2 ~2 PPL), EC-06 (full-rank control),
% EC-21 (k*). Theorem: appendix Weingarten/symmetric-square.
\begin{abstract}
Memory-efficient optimizers such as GaLore train large language models by projecting gradients
onto a rank-$r$ subspace that is recomputed every $T$ steps, on the assumption that this subspace
is a slowly drifting object that can be tracked. We show that beyond a small reproducible core,
there is no such object. Two estimates of the top-$r$ subspace computed \emph{at the same step}
from disjoint minibatches disagree as much as estimates computed $T$ steps apart ($0.73$ versus
$0.74$ of the maximal chordal distance $\sqrt{2r}$, measured at Pythia-160M with $r{=}128$): the
apparent rotation at every refresh is dominated by estimator noise. This holds across four model families in three
architecture classes and from 70M to 6.9B parameters, strengthening with scale, and more weakly in a vision transformer. Only ${\sim}39$ of $128$ directions are reproducible across minibatches, and averaging
cannot recover the rest: under $N$-fold averaging the gradient's spectral tail shrinks as
$N^{-1/4}$ rather than the $N^{-1/2}$ of pure noise, so no averaging budget makes the subspace
well defined. At Pythia-1B, a basis averaged over ten thousand steps still moves by ${\sim}84\%$ of the
geometric maximum per refresh, and sweeping the averaging strength leaves perplexity essentially unchanged (a $0.36$ perplexity span). What helps instead follows from treating each refresh as a change of
coordinates for Adam's state. Carrying the second moment blindly is provably about $(r{-}k^\star)/2$ worse
than the best rotation-blind estimator, while the first moment transports exactly through the
rotation, the optimal linear map under isotropic gradients and the rule LDAdam uses. In a
three-seed comparison at 1B and 40k steps, the full LDAdam update reaches $18.7$ perplexity at
the default $\beta_2{=}0.999$, beating untransported GaLore after its best $\beta_2$ fix ($19.3$).
Shortening the second-moment memory to $\beta_2{=}0.99$ helps the refreshing optimizers at their
reported recipes, by $1.8$ to $26$ perplexity, though for canonical GaLore the effect is small and
recipe-sensitive; a full-rank control reverses the preference. One measurable fact, subspace non-identifiability,
clarifies why GaLore works, which patches work, and what to check before trusting a low-rank
assumption: the reproducible rank $k^\star$.
\end{abstract}

% Intro v2 (final push). Story: reader believes subspace tracking; the same-step control removes
% the object; remedies follow. No em dashes. Numbers: EC-01/20/21/22/23/24/25 + spectrum cards.
\section{Introduction}

Training large models is increasingly bottlenecked by optimizer memory: Adam stores two moments
per parameter, doubling the footprint of the weights themselves. Low-rank gradient optimizers
such as GaLore~\citep{galore} sidestep this by projecting each gradient $G\in\RR^{m\times n}$
onto the top-$r$ left-singular subspace of $G$, running Adam in the resulting $r$-dimensional
space, and recomputing the subspace every $T$ steps. This matches
full-rank Adam at a fraction of the memory, and a family of refinements has grown around
it~\citep{ldadam,subtrackpp,frugal,golore,apollo}.

Underlying the whole approach is a geometric premise that is rarely stated and, to our knowledge,
was never measured: that the gradient's top-$r$ subspace is a meaningful, slowly varying object
that the refresh tracks. If the premise held, two consecutive estimates of the subspace would nearly
coincide, and carrying optimizer state across a refresh would be a small correction.

We measure the premise and find that beyond a small reproducible core, there is no such object.
The rotation between consecutive subspaces, measured as the chordal distance $\DeltaU$, saturates
its geometric maximum $\sqrt{2r}$ to $96.7$--$99.6\%$ at every refresh and at every rank
$r\in\{32,\dots,256\}$ during training (Pythia-1B~\citep{pythia}; 160M behaves the same): each refresh adopts a
nearly orthogonal frame. The decisive measurement is the one that removes time. On a pretrained
Pythia-160M checkpoint, two estimates of the top-$r$ subspace computed \emph{at the same step}
from disjoint minibatches disagree by $0.725\,\sqrt{2r}$, within a few percent of
the $0.742\,\sqrt{2r}$ the subspace rotates across a full refresh interval of $160$M training (a gap the
trajectory control of \S\ref{sec:noident:samestep} confirms is small at every training state). The
apparent rotation is not motion. It is estimator noise: of the $r{=}128$ directions GaLore
extracts, only $k^\star\approx 39$ are reproducible across two halves of one batch, and the rest
are re-drawn with the sample. (Both checkpoint values sit below the training-time saturation for
a reason the paper quantifies: a pretrained gradient's reproducible spike is stronger, and the
$k^\star$ shared directions pull the two estimates together.)

The cause is spectral, and it is not curable by averaging. The per-matrix gradient spectrum has a
small dominant spike, a few to ${\sim}20$ directions consistent with the tiny gradient subspace
documented for classification by \citet{gurari}, followed by a smooth tail with no gap anywhere near the ranks GaLore uses: at
$r{=}128$ the adjacent singular-value ratio is ${\approx}1.005$, the effective rank straddles the
cutoff, and the top-$r$ block holds only $46$--$68\%$ of the gradient's nuclear mass. A rank cut inside a
gapless bulk returns near-degenerate directions that resample freely (Davis--Kahan). Averaging
$N$ gradients does not open a gap: the deep spectral tail shrinks only as ${\approx}N^{-1/4}$
rather than the $N^{-1/2}$ of pure noise, because it is itself signal with a power-law spectrum,
so no averaging budget makes the top-$r$ frame well defined. GaLore works anyway, and the reason
is unglamorous: it captures a roughly constant fraction of gradient energy in a frame that is
re-drawn, not tracked (\S\ref{sec:energy}).

The redraw has a concrete cost, and the cost has remedies that we can now rank. Adam's second
moment is a slow average with memory $\tau_v=1/(1-\beta_2)\approx 1000$ steps at the default
$\beta_2{=}0.999$; a frame scrambled every $T\ll\tau_v$ steps leaves it chronically misaligned.
The obvious fix, stabilizing the basis by averaging the projection over time, fails exactly as
the geometry predicts: at Pythia-1B a basis averaged over ten thousand steps still moves by
$84\%$ of maximal per refresh (against $97\%$ with no averaging), and final perplexity is essentially flat across
four decades of averaging window (a $0.36$ perplexity span). What works is moving the state, not the frame. The first moment transports
exactly through the rotation $R=\Unew^\top\Uold$; the second moment transports through the
squared entries of $R$, and a carry that ignores $R$ is provably a factor ${\approx}(r{-}k^\star)/2$ worse
than the best blind alternative (\S\ref{sec:remedies}, Appendix~\ref{app:weingarten}). In a
three-seed controlled comparison at 1B, transported state (the LDAdam update, which combines
transport with error feedback) at the default $\beta_2$ reaches
$18.73$ perplexity, better than untransported GaLore after its best $\beta_2$ fix ($19.28$). The
cheaper remedy, shortening the memory to $\beta_2{=}0.99$, helps the refreshing optimizers we test
at their matched recipes, from $1.8$ to $26$ perplexity depending on optimizer and recipe (both the
size and, for canonical GaLore, the sign of the effect are recipe-dependent), while a full-rank
control with no refreshes reverses the preference, isolating the redraw as the cause.

\paragraph{Contributions.}
\begin{itemize}
  \item \textbf{The object is not there.} A same-step control showing that the GaLore subspace's
  apparent per-refresh rotation ($\DeltaU\approx\sqrt{2r}$, saturated $96.7$--$99.6\%$ across ranks
  and scales) equals the disagreement between two simultaneous estimates: the subspace is
  statistically non-identifiable beyond a small reproducible core of $k^\star\approx 39$ of $128$
  directions (\S\ref{sec:noident}). The same holds across four model families (three architecture classes), from 70M to 6.9B
  parameters, and in a vision transformer, strengthening with scale (\S\ref{sec:noident:general}).
  \item \textbf{Why, and why averaging cannot fix it.} The gradient spectrum is gapless at $r$,
  and its deep tail is power-law signal that shrinks under $N$-fold averaging only as
  ${\approx}N^{-1/4}$; $k^\star$ is predictable from the spectrum (within $8\%$ of the direct
  count, with the per-layer ordering matching measured energy retention), and no averaging budget
  opens a gap (\S\ref{sec:noident}). At 1B, ten-thousand-step basis averaging leaves both the
  realized rotation and the final perplexity essentially unchanged (\S\ref{sec:remedies}).
  \item \textbf{What to do instead, with theory and a controlled test.} The optimizer state
  should be transported through each refresh: the first moment by $R$, the second moment by the
  squared entries of $R$, with blind carry provably ${\approx}(r{-}k^\star)/2$ suboptimal
  (Appendix~\ref{app:weingarten}). A three-seed $2{\times}2$ at 1B/40K shows transport at the
  default $\beta_2$ beats the $\beta_2$ fix alone ($18.73$ vs $19.28$), and the two remedies are
  partially independent (\S\ref{sec:remedies}).
  \item \textbf{A prescription and a probe.} Lower $\beta_2$ to $0.99$ for any optimizer that
  refreshes a gradient subspace (validated for canonical GaLore and LDAdam, with APOLLO-Mini as a
  directional consistency check and a full-rank control reversing the preference), and measure
  $k^\star$ before trusting a rank: if $r>k^\star$, the excess directions are noise being
  refreshed (\S\ref{sec:remedies}).
\end{itemize}
The object of study is the gradient subspace these methods compute, not a LoRA-style weight
parameterization (\S\ref{sec:setup}).

% Setup — GaLore-family object, two Adam clocks, ΔU = Grassmannian chordal distance.
% Ported/condensed from proofs/setup.tex + L1. Scoped to GaLore-family.
\section{Setup and notation}
\label{sec:setup}

\paragraph{The GaLore-family optimizer.}
We study optimizers that train a weight matrix $W\in\RR^{m\times n}$ by running Adam inside a
periodically-recomputed low-rank subspace of its gradient. At step $t$ the gradient is
$G_t = \nabla_W \mathcal{L}\in\RR^{m\times n}$. Every $T$ steps the optimizer computes the top-$r$ left
singular subspace of the current gradient, $U_t\in\RR^{m\times r}$ with $U_t^\top U_t = I_r$ (the columns
are the leading left singular vectors of $G_t$), and holds it fixed for the next $T$ steps. Within an
interval it projects each gradient into the subspace, $\tilde G = U^\top G\in\RR^{r\times n}$, applies the
Adam update in that $r$-dimensional space, and maps the result back: $W \leftarrow W - \eta\, U\,\mathrm{Adam}(\tilde G)$.
This is GaLore~\citep{galore}; the same template, \emph{explicitly form a top-$r$ gradient subspace, refresh it
every $T$ steps}, underlies the family that has grown around it~\citep{ldadam,subtrackpp,frugal,golore,apollo}.
Our claims are restricted to this family. LoRA and other low-rank \emph{parameterizations}, which impose a
low-rank structure on the \emph{update} and never compute a gradient subspace, are out of scope.

\paragraph{Adam's two clocks.}
Adam~\citep{adam} carries two exponential moving averages of the projected gradient: the first moment
$m_t = \beta_1 m_{t-1} + (1{-}\beta_1)\tilde g_t$ and the second moment
$v_t = \beta_2 v_{t-1} + (1{-}\beta_2)\tilde g_t^{\,2}$. Each is a low-pass filter with a relaxation time
set by its decay,
\begin{equation}
  \tau_m = \frac{1}{1-\beta_1}\approx 10,\qquad
  \tau_v = \frac{1}{1-\beta_2}\approx 1000\quad(\beta_1{=}0.9,\ \beta_2{=}0.999),
  \label{eq:clocks}
\end{equation}
so the second moment forgets the past $100\times$ more slowly than the first. These two timescales are the
constants against which the refresh interval $T$ must be compared, and the asymmetry between them is the
source of the staleness we analyze in \S\ref{sec:staleness}.

\paragraph{Measuring subspace rotation.}
To quantify how much the basis changes at a refresh we compare the bases before and after,
$\Uold$ and $\Unew\in\RR^{m\times r}$, through their $r$ principal angles $\theta_1,\dots,\theta_r$
(the angles between $\mathrm{range}(\Uold)$ and $\mathrm{range}(\Unew)$). Singular vectors carry an
arbitrary per-column sign, so the raw Frobenius distance $\|\Unew-\Uold\|_F$ is dominated by sign flips
rather than rotation; we align signs first (Appendix~\ref{app:deltau}) and report the sign-corrected
distance
\begin{equation}
  \DeltaU \;=\; \Bigl(2r - 2\textstyle\sum_{i=1}^{r}\cos\theta_i\Bigr)^{1/2},
  \qquad 0 \le \DeltaU \le \sqrt{2r}.
  \label{eq:deltaU}
\end{equation}
This is the chordal (principal-angle) Frobenius distance between orthonormal frames~\citep{grassmannhandbook}: it is $0$ when the
subspace is unchanged and attains its maximum $\sqrt{2r}$ when the two bases are orthogonal
($\theta_i=\pi/2$ for all $i$), i.e.\ when the new subspace shares no direction with the old. We also track the
fraction of gradient energy the subspace captures,
\begin{equation}
  \gamg = \frac{\|UU^\top G\|_F^2}{\|G\|_F^2}\;\in\;[0,1],
  \label{eq:gamma}
\end{equation}
which is what the projection actually preserves and the quantity we return to in \S\ref{sec:energy}.

% §The subspace is not identifiable — replaces §Tumble.
% Order: (1) observed rotation law EC-01; (2) same-step control EC-20 (the reframe);
% (3) no-gap spectrum + Haar/Weingarten (reused from old tumble:gap); (4) k* validation EC-21;
% (5) budget-relative identifiability EC-22 + F2; (6) forward pointer to remedies EC-24.
% Every number traces to a card. Scope: GaLore-family only. No em dashes, no banned words.
\section{The subspace is not identifiable}
\label{sec:noident}
\label{sec:tumble}% alias: legacy cross-references

The subspace GaLore projects onto is recomputed every $T$ steps, and it moves. This section establishes what that
motion is. The measured per-refresh basis change saturates the geometric maximum at every rank and scale
(\S\ref{sec:noident:law}). A same-step control then locates the cause: two top-$r$ estimates taken at the identical
step from disjoint minibatches already disagree by as much as two estimates taken $T$ steps apart
(\S\ref{sec:noident:samestep}). The refresh is therefore not the motion of a stable object. The top-$r$ subspace of
a minibatch gradient is statistically non-identifiable beyond a small reproducible core, at the ranks these
methods use. The rest of the
section shows why: the gradient spectrum has no gap at $r$ (\S\ref{sec:noident:gap}); the identifiable part is small
but measurable and consistent (\S\ref{sec:noident:kstar}); and averaging does not manufacture identifiability
(\S\ref{sec:noident:budget}).

\subsection{A scale-invariant rotation law}
\label{sec:noident:law}

We measure $\DeltaU$ at every refresh during training and find that the subspace rotates almost as far as it
geometrically can. Figure~\ref{fig:deltaU} reports the measured rotation across an $8\times$ range of ranks at
Pythia-1B (exact values in Appendix~\ref{app:deltau}): in every case $\DeltaU$ reaches $96.7$--$99.6\%$ of its maximum
$\sqrt{2r}$, the value attained when the new basis is orthogonal to the old. The dependence on rank is $\DeltaU\approx\sqrt{2r}$: for instance
$\DeltaU(r{=}256)/\DeltaU(r{=}32)=22.10/7.93=2.79$, close to $\sqrt{8}=2.83$. The same holds from Pythia-160M to
1B.
% [RESOLVED: scale range 160M-1B is carded (EC-01); 8B/12B dU were cut as unverified (CUT-04);
%  the ratio law is measured 160M-1B only, so no "across corpora" claim for dU.]
At this saturation the mean cosine, $1-\DeltaU^2/2r$, is below $0.07$: at each refresh the optimizer discards
the frame it was using and adopts a nearly orthogonal one (mean angle $86$--$90^\circ$). (On a pretrained 160M
checkpoint the same redraw sits lower, at $0.74\,\sqrt{2r}$, a mean angle of ${\approx}63^\circ$: a pretrained
gradient's reproducible spike is stronger, and its $k^\star\!\approx\!39$ shared directions pull consecutive
estimates together; \S\ref{sec:noident:samestep}.) This is not slow drift. It is turnover.

\begin{figure}[t]
\begin{center}
\includegraphics[width=0.52\linewidth]{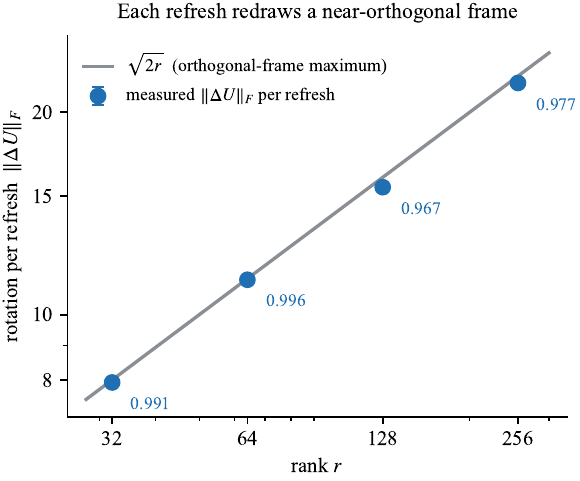}
\end{center}
\caption{Measured subspace rotation per refresh against rank (Pythia-1B, $T{=}50$; log--log). The points sit on the
$\sqrt{2r}$ ceiling, the maximum possible Frobenius distance between two rank-$r$ frames, at every rank (annotated
ratios $0.967$--$0.996$). Each refresh returns a nearly independent subspace. Exact values and the sign correction:
Appendix~\ref{app:deltau}.}
\label{fig:deltaU}
\end{figure}

\subsection{The rotation is estimator noise}
\label{sec:noident:samestep}
\label{sec:tumble:wrongquantity}% alias: legacy cross-references (prior-stability-reports material now lives here)

A near-maximal $\DeltaU$ across a refresh has an obvious reading: the subspace genuinely rotates by a large angle in
$T$ steps. We test that reading with a control that removes time. At a single step we split the batch into two
disjoint halves, compute the top-$r$ subspace of each half's gradient, and measure the distance between the two
(protocol in Appendix~\ref{app:probes}).
Nothing has changed between the two estimates except the sample. If the top-$r$ subspace were a well-defined
property of the gradient, the two same-step estimates would agree.

They do not. Averaged over four layers at Pythia-160M with $r{=}128$, the same-step split-batch disagreement is
$0.725\,\sqrt{2r}$. The across-time rotation over a full refresh interval of training, at the same $160$M scale, is
$0.742\,\sqrt{2r}$ (Figure~\ref{fig:samestep}a), essentially the same-step value: the across-time rotation
exceeds the same-step floor by at most a few percent of $\sqrt{2r}$. That floor is not an artifact of the
pretrained checkpoint at which the control is run: repeated at five points along Pythia-160M training (steps
$1$K to $143$K), the four-layer same-step disagreement stays in $0.64$--$0.71\,\sqrt{2r}$
(Appendix~\ref{app:probes}), lowest early and rising toward the pretrained value, so at every training state
the bulk of the across-time rotation is already present with no time elapsed. The distance a subspace travels in
$T$ steps is essentially the distance two estimates of it already sit apart at a single step, so the ``rotation''
carries little information about elapsed time: it is dominated by estimator noise, not by motion of a stable
object. A direct count confirms the same fact at
the level of individual directions: averaged over the four layers, only $k^\star\!\approx\!39$ of the $r{=}128$
directions the top-$r$ cut returns are reproducible across the split ($48$ on the query-key-value matrix, with
standard error below $1$ over repetitions), and the rest reshuffle with the sample.

This also settles an apparent conflict with reports that the GaLore subspace is stable. Those reports measure
quantities that stay small while $\DeltaU$ is near-maximal. The rank-$r$ subspace has $r$ principal angles, and
many can remain small while the tail angles near $90^\circ$ dominate the chordal distance, so an angle distribution
that looks mostly stable, without conversion to $\DeltaU$, still corresponds to a near-maximal rotation~\citep{frugal}. An incremental tracker is built to move its
estimate slowly, so its self-consistency across steps is a property of the tracker, not of the subspace a fresh SVD
would return~\citep{subtrackpp}. And a raw singular-vector overlap conflates the arbitrary per-column sign of the SVD
with genuine rotation; once signs are aligned (\S\ref{sec:setup}), the rotation is the near-maximal value
of Figure~\ref{fig:deltaU}. None of these is the sign-corrected frame distance $\DeltaU$, which is the quantity that
governs the carried optimizer state.

\subsection{Why: there is no spectral gap at rank \texorpdfstring{$r$}{r}}
\label{sec:noident:gap}
\label{sec:tumble:gap}% alias: legacy cross-references

The recomputed top-$r$ subspace can be reproducible across samples only if the rank-$r$ cut separates
well-determined directions from the rest, that is, if the spectrum has a gap at index $r$. There is not.
Figure~\ref{fig:spectrum} shows the per-matrix gradient spectrum at Pythia-160M and 1B: a small dominant spike,
then a smooth, near-isotropic tail that the rank-$128$ cutoff slices straight through. Across layers, steps, and
both scales:
\begin{itemize}[leftmargin=1.5em,itemsep=1pt,topsep=2pt]
  \item the dominant spike is small: a few to ${\sim}20$ directions hold the first half of the gradient's
  squared singular-value (Frobenius) mass, consistent with the tiny gradient subspace documented for classification by \citet{gurari};
  \item there is no gap at the cutoff: the adjacent singular-value ratio at $r{=}128$ is $\approx 1.005$ on average,
  and no jump anywhere in the bulk exceeds ${\sim}3\%$ ($\max$ ratio $1.011$--$1.030$);
  \item the effective rank (entropy of the normalized spectrum) straddles the cutoff, ranging $89$--$358$ across
  layers and both scales, so $r{=}128$ sits inside the flat tail, not at a boundary;
  \item the top-$r$ block captures only $46$--$68\%$ of the gradient's singular-value (nuclear) mass, leaving
  roughly a third to over half outside the subspace at all times. (The spike dominates the squared mass while the long flat
  tail dominates the nuclear mass; the contrast between the two norms is itself a signature of the tail's
  flatness.)
\end{itemize}
This agrees with independent measurements: \citet{gurari} establish the small dominant spike directly, and a
concurrent large-model analysis reports the same spike-tail profile with the spike at ${\sim}1.5\%$ of
directions~\citep{spectra}, a separation at the spike rather than at $r{=}128$. The
complementary evidence from the randomization side, that past the spike the SVD frame carries little more than a
random one, is taken up in \S\ref{sec:energy} and \S\ref{sec:related}.

\begin{figure}[t]
\begin{center}
\includegraphics[width=0.95\linewidth]{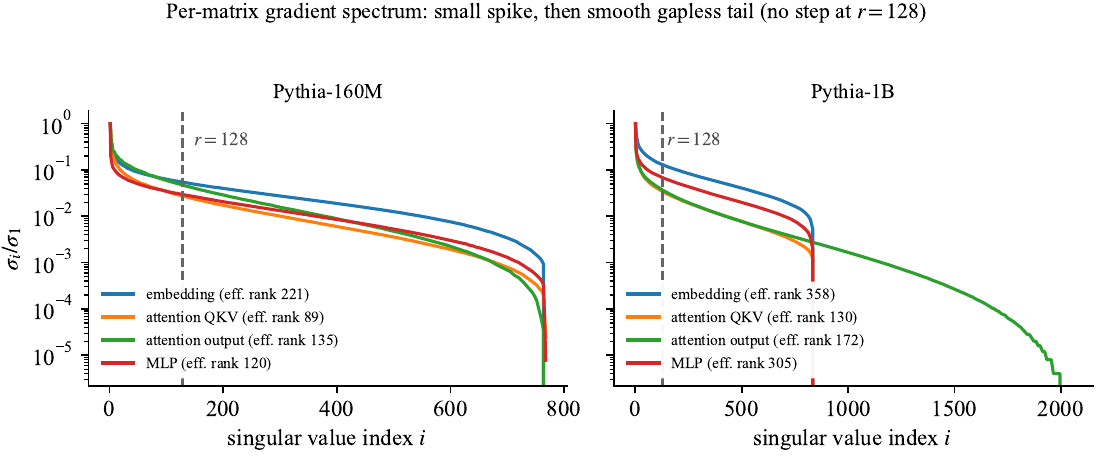}
\end{center}
\caption{Per-matrix gradient singular spectrum at Pythia-160M (left) and 1B (right): the reason the subspace is not
identifiable. A small dominant spike (a few to ${\sim}20$ directions holding half the squared mass) gives way to a
smooth, near-isotropic tail with no gap. The GaLore cutoff $r{=}128$ (dashed) falls inside the flat region, where
the adjacent singular-value ratio is $\approx 1.005$ and the top-$r$ block holds only $46$--$68\%$ of the gradient's
nuclear mass. The rank-$r$ subspace is not separated from its complement.}
\label{fig:spectrum}
\end{figure}

\paragraph{From no gap to Haar-random redraw.}
When singular directions are near-degenerate they are maximally sensitive to perturbation: the Davis--Kahan and
Wedin $\sin\theta$ theorems bound a subspace's rotation by the perturbation size over the gap that isolates
it~\citep{davis_kahan,wedin}. As the gap at $r$ vanishes, that bound diverges, and the SVD is free to return any
frame within the degenerate block. Sub-threshold directions therefore behave as a Haar-random draw within the
tail, whether the perturbation is a step of training or a change of minibatch. This is why two same-step estimates
disagree as much as two across-time estimates (\S\ref{sec:noident:samestep}). The same model predicts the redraw
distance approaches the $\sqrt{2r}$ ceiling as $r/d\to 0$ (Appendix~\ref{app:deltau_scaling}). At 160M, where
$r/d=1/6$, it stays below the ceiling, matching the measured $0.74\,\sqrt{2r}$. The Weingarten lower bound of \S\ref{sec:staleness}
follows for the carried second moment under the same Haar model (Theorem~\ref{thm:blind_lower_bound},
Appendix~\ref{app:weingarten}). The stable part of the geometry
is real but small: the spike of \citet{gurari} is genuine, and GaLore's $r{=}128$ overshoots it by an order of
magnitude, reaching past the determined directions into the gapless bulk that redraws.

\subsection{The identifiable part is small, measurable, and consistent}
\label{sec:noident:kstar}

The same-step count gives a direct measure of how many directions survive resampling: a four-layer mean of
$k^\star\!\approx\!39$ of $128$, ranging from $48$ on the query-key-value matrix down to $31$ on the attention
output, with a standard error below $1$ over repetitions, so the count is stable rather than noise. Reading the
same data a second way reproduces the pattern: inverting the measured spectral shape through a detectability
threshold fixed a priori, with one calibrated noise scale per layer and no shape parameters, recovers the same
layer ordering, $\text{qkv} > \text{mlp} > \text{embed} > \text{dense}$ (inversion $42,37,19,17$ against the
direct $48,43,34,31$). Both readings derive from the same measurement, so this is a consistency check between two
estimators rather than an independent prediction; both place the reproducible core well below $r{=}128$. The
ordering matches the measured energy the top-$r$ block retains in each layer ($73$--$81\%$ for qkv and mlp versus
$42$--$48\%$ for embed and dense): layers whose gradient concentrates more have more reproducible directions.
$k^\star$ is thus a measurable property, not an artifact of one estimator.

\subsection{Averaging does not rescue it}
\label{sec:noident:budget}

Identifiability might be recovered by averaging: with a large enough gradient sample the noisy tail would shrink and
a gap could open at $r$. It does not shrink fast enough. Under $N$-fold gradient averaging, pure sampling noise
falls as $N^{-1/2}$, so a noise singular value scaled by $\sqrt{N}$ would stay flat. The deep spectral tail instead
shrinks only as $\approx N^{-1/4}$: across all four probed layers the scaled tail value grows by $2.0$--$3.2\times$
from $N{=}1$ to $N{=}64$ (Figure~\ref{fig:samestep}). The tail is not noise being averaged away. It is signal, with
a de-censored power-law spectrum of exponent $\alpha=1.03$--$1.35$, obtained two independent ways (a direct tail fit
and the overlap-versus-$N$ curve). A shuffled-document control, which breaks within-document correlation, retains
most of the anomaly ($1.7$--$2.2\times$ against the $2.0$--$3.2\times$ of the true pool), so the persistent tail is
mostly genuine gradient signal, with only a minor part attributable to within-document correlation.

We state the limits of this claim plainly. We do not claim a closed-form law for how the subspace overlap grows
with $N$: the naive spiked-covariance prediction fails out of sample (mean absolute error $0.073$), so we report
only the measured scaling, not a formula. That scaling is enough for the conclusion. Because the tail is signal that
decays as a power law and shrinks under averaging only as $N^{-1/4}$, no finite averaging budget opens a gap at $r$,
and no rank $r$ lands on a gap that is not there. The top-$r$ subspace GaLore uses is not well-defined, and no averaging budget we test recovers it; this holds from 160M to 2.8B and across GPU backends (Appendix~\ref{app:probes}).

\begin{figure}[t]
% BUILD FROM: /Users/noel.thomas/BPDR/icml_27/analysis/data/samestep2_160m.jsonl
%   (fields per record: N in {1,2,4,8,16,32,64}; layer in {embed_in, attention.dense,
%    attention.query_key_value, mlp.dense_h_to_4h}; dU_over_sqrt2r; overlap; spec_vals (singular
%    spectrum); spec_idx). Fit/de-censoring done by icml_27/analysis/g3_stage2.py.
% PANEL (a) — averaging does not close the gap: x = N (log2 axis), y = same-step split-batch
%   dU_over_sqrt2r, one line per layer (4 layers). Add a horizontal dashed reference at 0.742
%   (across-time rotation, EC-20) and mark the N=1 four-layer mean 0.725. Curves stay near the
%   reference and fall only slowly with N.
% PANEL (b) — the tail is signal, not noise: for the deep spectral tail (flat region well past
%   r=128), plot the tail singular value scaled by sqrt(N), s*sqrt(N), versus N in log-log, one
%   line per layer. A pure-noise tail would be flat (s ~ N^-1/2); the observed lines rise by
%   2.0-3.2x from N=1 to 64 (slope ~ +1/4, i.e. s ~ N^-1/4). Overlay the fitted de-censored
%   signal power law with exponent alpha in [1.03, 1.35].
\begin{center}
\includegraphics[width=0.95\linewidth]{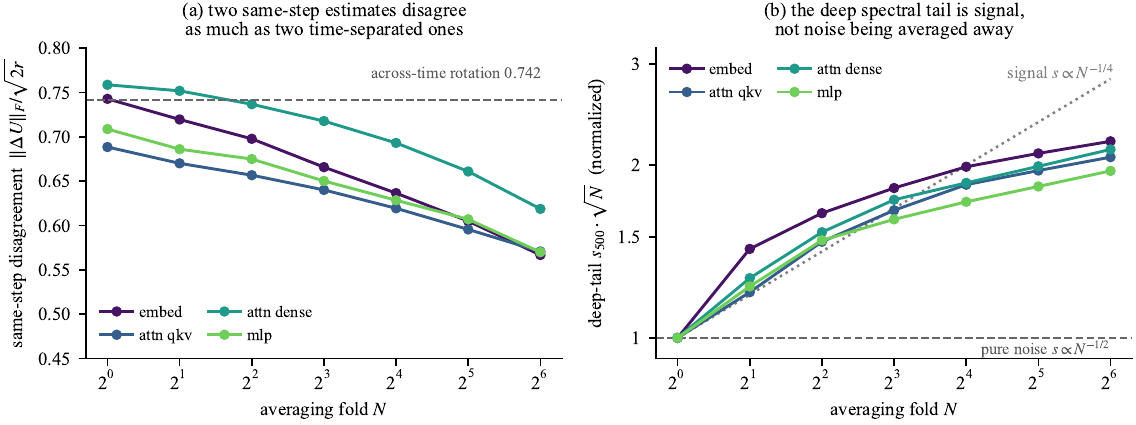}
\end{center}
\caption{Averaging cannot make the top-$r$ subspace identifiable (Pythia-160M, $r{=}128$, four layers).
\textbf{(a)} Same-step split-batch subspace disagreement $\DeltaU/\sqrt{2r}$ against averaging fold $N$. At $N{=}1$
it is $0.725$, indistinguishable from the across-time rotation $0.742$ (dashed), and it falls only slowly with $N$.
\textbf{(b)} The deep tail singular value scaled by $\sqrt{N}$. Pure noise would be flat; the measured tail grows
$2.0$--$3.2\times$ from $N{=}1$ to $64$, so it shrinks as $\approx N^{-1/4}$, not $N^{-1/2}$. The tail is signal
with power-law exponent $\alpha=1.03$--$1.35$, so no averaging budget opens a gap at $r$.}
\label{fig:samestep}
\end{figure}

\subsection{The non-identifiability is general}
\label{sec:noident:general}

The same-step control is measured at Pythia-160M, but the effect is not specific to Pythia, to language, or to
GaLore. We ran the same split-batch probe on eleven pretrained checkpoints, all trained with full-rank Adam: seven
Pythia scales (70M--6.9B), GPT-2-large, Qwen2.5-3B, Llama-3-8B, and a vision transformer (ViT-B) probed on
CIFAR-10. In every language model the same-step disagreement is $0.67$--$0.83$ of the maximum $\sqrt{2r}$, with no
spectral gap at $r$ (mean adjacent ratio $1.004$--$1.009$); it rises with scale
(Figure~\ref{fig:generality}) as the reproducible core shrinks, from $k^\star\!\approx\!50$ at 70M to
$k^\star\!\approx\!21$ at 6.9B. The effect crosses architecture classes (GPT-NeoX, GPT-2, and the Llama-class
Qwen2.5 and Llama-3) and holds, more weakly, in the vision transformer ($0.53$ of maximal on CIFAR). Because every
model here is trained with full-rank Adam, the non-identifiability is a property of the gradient itself, not of the
low-rank projection: the top-$r$ subspace any method would extract is estimator noise past a small core, at every
scale we measured. Per-model numbers and the protocol are in Appendix~\ref{app:generality}.

\begin{figure}[t]
\begin{center}
\includegraphics[width=0.62\linewidth]{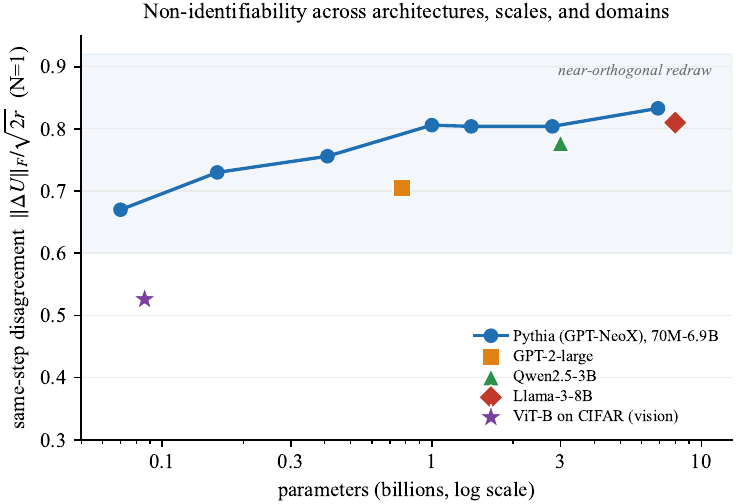}
\end{center}
\caption{Same-step subspace disagreement $\DeltaU/\sqrt{2r}$ at $N{=}1$ (split-batch, pretrained checkpoints,
$r{=}128$, mean over four layer types) across architectures, scales, and domains. It rises with scale in Pythia
(70M--6.9B), sits in the same band for GPT-2-large, Qwen2.5-3B, and Llama-3-8B, and holds more weakly for a vision
transformer on CIFAR. Every model is full-rank-Adam-pretrained, so the non-identifiability is a property of the
gradient, not of GaLore; no spectral gap at $r$ appears in any of them.}
\label{fig:generality}
\end{figure}

\paragraph{Forward pointer.}
This scaling is measured at Pythia-160M, where the tail is already flat at $r{=}128$. At 1B, where $r{=}128$ lies
deeper in the tail of a wider gradient, it predicts that averaging successive bases across refreshes cannot
stabilize the subspace either. We verify this directly, after the energy interlude of \S\ref{sec:energy}, with the
second-moment remedies at 1B in \S\ref{sec:staleness}: basis averaging leaves the per-refresh rotation near its
maximum, while the two remedies that
do help act on the carried optimizer state rather than on the subspace.

% §Energy — why GaLore works anyway. MODEST / qualitative. Coverage prop (L3) + GoLore corroboration.
\section{Why GaLore still works}
\label{sec:energy}

If the subspace it projects onto is replaced by a nearly orthogonal one every $T$ steps, why does GaLore train as
well as it does? The answer is that the projection's job is not to track a fixed subspace but to capture gradient
energy, and how much it captures barely depends on \emph{which} subspace it picks. The descent contribution of a
projection is governed by the fraction of gradient energy it retains, $\gamg$ (Eq.~\ref{eq:gamma}). For a subspace
$U$ spanning $r$ directions of the gradient second-moment, the expected coverage is the cumulative spectral mass
it covers,
\begin{equation}
  \EE[\gamg] \;=\; \frac{\sum_{i=1}^{r}\lambda_i}{\sum_{i=1}^{n}\lambda_i},
  \label{eq:coverage}
\end{equation}
where $\lambda_1\ge\cdots\ge\lambda_n$ are the spectrum's energies (Appendix~\ref{app:coverage}). This identity is
an expectation over the gradient covariance; the per-matrix SVD energy fractions reported below are the
instantaneous quantity it matches in expectation. In the presence
of a sharp gap, only the gap-defined subspace captures the dominant mass and a rotation away from it would be
costly. But \S\ref{sec:tumble:gap} showed the tail past the small spike is smooth and globally
near-isotropic, with no gap anywhere in the bulk rather than merely at $r$; it is this global flatness, not the
absence of a single gap at $r$, that the argument needs. In that regime Eq.~\ref{eq:coverage} is nearly flat in
the choice of subspace: any $r$ directions that include the spike and otherwise sample the tail capture a
comparable fraction. The freshly-recomputed basis,
near-orthogonal to its predecessor, is therefore an essentially equivalent projector: measured during 1B
training with $r{=}128$, GaLore retains a roughly constant $\gamg\approx 0.67$--$0.73$ of the squared gradient
throughout, regardless of the redraw. (The nuclear fraction is lower, $46$--$68\%$, because the flat tail
carries more of the nuclear mass; \S\ref{sec:tumble:gap}.) \citet{golore}
reach a compatible conclusion from the convergence side: once the true gradient shrinks relative to its noise, the
SVD-selected subspace loses its advantage over a random one, and a projection chosen ``without any preference''
suffices in its place. Our diagnostic also reconciles a tension in more recent GaLore-family work.
\citet{golore} and a randomized Grassmann variant~\citep{randsubspaces2510} replace the exact SVD projection
with a random one and report training at least as good as GaLore, the latter arguing that flat curvature makes
the precise subspace unnecessary; GaLore-2~\citep{galore2} instead finds that a purely random projection
degrades quality and retains a periodic SVD in a hybrid schedule. The reproducible rank $k^\star$ reconciles
them: the exact SVD earns its cost only to the extent the subspace is identifiable, and is redundant once
$k^\star\ll r$. Because the subspace grows \emph{less} identifiable as training proceeds (the same-step floor
rises across the trajectory; Appendix~\ref{app:probes}), a random projection is safest late in training, while
any residual advantage of the exact SVD is concentrated early. This is a quantitative form of the
flat-curvature intuition, and of the boundary the hybrid schedule approximates.
Low-rank projection is known to introduce a per-step bias relative to full Adam~\citep{golore,gum};
our account is orthogonal, since we do not claim the projection is unbiased, only that the \emph{choice} of which $r$ directions
it selects barely changes the energy captured.

This reframes what GaLore is. It is not tracking a persistent low-rank structure in the gradient; there is none
to track at rank $r$. It is a rotating energy-capture projector: a frame that turns over completely at every
refresh while preserving the energy fraction that drives descent. The descent direction survives the redraw.
What does not survive is the optimizer state carried across the rotation, and, as we show next, the second moment
in particular pays for it.

% §What works and what does not. Replaces the old §Staleness.
% Survivors (rewritten to the one idea): the full β₂ validation (LDAdam EC-03, canonical length caveat
% EC-04, APOLLO EC-05) and the full-rank causal control (EC-06). New: averaging fails (EC-23/24),
% transport 2×2 headline table (EC-25), gasket-seed β₂ (EC-26). Numbers trace to INVENTORY cards.
\section{What works and what does not}
\label{sec:remedies}
\label{sec:staleness}% alias: incoming cross-references from other sections target the old staleness label

If the rank-$r$ frame that GaLore projects onto is redrawn nearly orthogonally at every refresh
(\S\ref{sec:tumble}), the optimizer state carried across that refresh points at a target that no longer
exists. The first moment survives the move: it is a vector in the subspace and transports exactly by the
rotation $R=\Unew^\top\Uold$ (Proposition~\ref{prop:mmse}), the change of basis that LDAdam~\citep{ldadam}
and SUMO~\citep{sumo} apply. The second moment does not. It is a per-coordinate variance whose optimal new-basis value passes
through the \emph{squared} entries of $R$, and a carry blind to $R$ sits a factor $\approx (r{-}k^\star)/2$ above the best
any blind estimator can reach (Theorem~\ref{thm:blind_lower_bound}). With memory
$\tau_v=1/(1-\beta_2)\approx 1000$ steps at the default and refreshes every $T\approx 50$--$200$ steps, the
second moment averages statistics from many dead subspaces and is chronically misaligned; its lag grows with
$\tau_v$ on the same curve as full-rank Adam (Appendix~\ref{app:overlay}). Three remedies
present themselves: average the basis so it stops moving, transport the state through the rotation, or shorten
the second moment's memory. We test all three at Pythia-1B. Averaging fails; transport and a shorter memory
both work, and we measure why.

\subsection{Averaging the basis does not work}
\label{sec:remedies:averaging}

The direct fix for a subspace that moves too much is to stop it from moving. We build the EMA-basis variant:
instead of the top-$r$ singular subspace of the current gradient, the basis is the top-$r$ eigenvectors of an
exponential moving average of $GG^\top$ with decay $\beta_3$, so that $\beta_3{=}0$ recovers GaLore exactly
and larger $\beta_3$ averages the subspace over a growing window. At Pythia-1B trained from scratch on WikiText~\citep{merity2017pointer}
($r{=}128$, $T{=}160$, learning rate $4\times10^{-3}$, warmup $4000$, $40$K steps, AMD MI210, one seed), the
final perplexity does not move with $\beta_3$. At the default $\beta_2{=}0.999$ the row is flat:
$\{21.088,\,21.025,\,21.074,\,21.026,\,21.385\}$ for $\beta_3=\{0,\,0.9,\,0.99,\,0.999,\,0.9999\}$, spanning
$0.36$ across four decades of averaging window ($0.06$ excluding the longest, $10^4$-step window, which alone
sits $0.36$ above the minimum). The flat row replicates
on a second cluster (NVIDIA L40S, same recipe): $20.39$ at $\beta_3{=}0$ against $20.52$ at $\beta_3{=}0.99$
(we compare only within a cluster, never across).

The measurement of why is direct. Averaging the covariance does slow the basis, but not enough to matter: late
in training (steps $\ge 20$K) the mean realized rotation is $\DeltaU = 15.48,\,14.51,\,13.42$ for
$\beta_3 = 0,\,0.999,\,0.9999$, against the ceiling $\sqrt{2r}=16.0$. Even a $10{,}000$-step average of
$GG^\top$ leaves the refresh-to-refresh basis change at $84\%$ of maximal (against $97\%$ with no averaging). Basis averaging cannot stabilize
what is not identifiable: past the dominant spike the spectrum is gapless (\S\ref{sec:tumble:gap}), so no amount
of temporal averaging carves a stable rank-$128$ frame out of the isotropic tail.

\subsection{Transporting the state works}
\label{sec:remedies:transport}

If the basis will not hold still, move the state with it. Table~\ref{tab:transport} crosses the two arms that
matter, $\{$carry, transport$\}$ against $\{\beta_2{=}0.999,\ \beta_2{=}0.99\}$, three seeds per cell at
Pythia-1B/$40$K. This table uses warmup $1440$; its carry-$\beta_2{=}0.999$ cell ($22.07$) therefore sits about
$0.8$--$1.0$ PPL above the same nominal configuration in the averaging and gasket experiments (warmup $4000$: $21.09$
at one seed, \S\ref{sec:remedies:averaging}, and $21.27$ over three seeds, \S\ref{sec:remedies:beta2}), purely
through the warmup difference. Magnitudes shift between recipes; directions do not. We compare $\beta_2$ and state
policies only within a fixed recipe, never perplexities across recipes. Read the table three ways, in order. First, transporting the state at the default $\beta_2{=}0.999$
reaches $18.73$, better than lowering $\beta_2$ without transport ($19.28$): moving the state is worth more than
the hyperparameter fix on its own. Second, the second-moment cost of the wrong $\beta_2$ under a plain carry is
large, $22.07$ against $19.28$: a $2.79$ PPL gap, about twenty times the standard error of the
difference (two-sample $t\approx 19$ on $4$ degrees of freedom). Third, that gap does not vanish once the
state is transported: it shrinks to $1.81$ but persists; the two remedies are therefore partially independent
rather than substitutes. (The transport row is the LDAdam update, so this $1.81$ is the same measurement that
recurs below as LDAdam's matched-$40$K $\beta_2$ gap, \S\ref{sec:remedies:beta2}.) One caveat is mandatory: the transport arm is the full LDAdam update, which applies the
change-of-basis transport of both moments together with error feedback, and this experiment does not separate
the contribution of transport from that of error feedback. The transport rule itself is LDAdam's; our
contribution here is the optimality statement behind it (Proposition~\ref{prop:mmse},
Theorem~\ref{thm:blind_lower_bound}) and the controlled $2\times2$ comparison, not the rule.

\begin{table}[t]
\caption{Final validation perplexity at Pythia-1B under the two independent remedies: transporting the
optimizer state across each refresh and shortening the second moment's memory. Rows are the state policy at a
refresh; columns are $\beta_2$. Each cell is mean$\pm$standard deviation over three seeds. Recipe (identical
across cells): Pythia-1B trained from scratch on WikiText, rank $r{=}128$, refresh interval $T{=}160$, learning
rate $4\times10^{-3}$, warmup $1440$, $40$K steps. ``Carry'' is canonical GaLore (the second moment is carried
unchanged across refreshes); ``transport'' is the LDAdam update~\citep{ldadam} (change-of-basis transport of
both moments plus error feedback). Transporting at the default $\beta_2{=}0.999$ ($18.73$) already beats lowering
$\beta_2$ alone ($19.28$); the carry $\beta_2$ gap is $2.79$, about twenty times the standard error of the
difference, and persists under transport ($1.81$). Transport here is the full LDAdam update; the table does not separate state transport from error
feedback.}
\label{tab:transport}
\begin{center}
\begin{tabular}{lcc}
\toprule
State at refresh & $\beta_2{=}0.999$ (default) & $\beta_2{=}0.99$ \\
\midrule
Carry (canonical GaLore)        & $22.07\pm0.17$ & $19.28\pm0.18$ \\
Transport (LDAdam rule)         & $18.73\pm0.13$ & $16.92\pm0.16$ \\
\bottomrule
\end{tabular}
\end{center}
\end{table}

\subsection{A shorter memory works, with a recipe caveat}
\label{sec:remedies:beta2}
\label{sec:staleness:val}% alias: appendix/conclusion reference the old validation-subsection label

The second remedy, lowering $\beta_2$ from $0.999$ to $0.99$, is the cheaper of the two, and it helps the
refreshing optimizers we test at their matched recipes. For canonical GaLore the effect is small (about $2$
perplexity) and hardware-sensitive: it is three-seed robust on our primary cluster, while a matched-recipe run on a
different GPU backend does not reproduce it, consistent with the projection's periodic SVD being ill-conditioned
and resolving differently across backends. We therefore present the shorter memory as the secondary, less robust
remedy, behind state transport. The forgetting-factor reason is standard~\citep{guoljung}: when the target is
reset every $T\ll\tau_v$ steps, a long memory keeps integrating pre-rotation statistics that are now bias, and
the error-minimizing forgetting factor drops below the no-reset optimum. A bias--variance argument specialized to
a self-reset target makes the direction precise (Appendix~\ref{app:forgetting}): the minimizer is interior and
moves down as the refresh interval shrinks. Empirically the $\beta_2$ gain does not
depend on the averaging window of the previous experiment: across the $\beta_3$ grid the gap between
$\beta_2{=}0.999$ and $\beta_2{=}0.99$ is $1.94,\,2.04,\,2.16,\,2.27,\,2.13$ PPL at
$\beta_3=0,\,0.9,\,0.99,\,0.999,\,0.9999$ (seed 0), so averaging neither substitutes for the memory fix nor
interferes with it. At the gasket recipe repeated over seeds (same recipe, warmup $4000$), canonical GaLore
gives $21.27\pm0.17$ at $\beta_2{=}0.999$ against $19.04\pm0.13$ at $\beta_2{=}0.99$ (three seeds each), a
$2.23$ PPL gap.

\paragraph{Across the family.} The direction holds for every optimizer we tried
(Figure~\ref{fig:beta2}; exact values and configurations in Appendix~\ref{app:configs},
Table~\ref{tab:beta2}). For LDAdam, whose error feedback makes the second moment's calibration load-bearing,
lowering $\beta_2$ from $0.999$ to $0.99$ improves perplexity from $57.69\pm0.11$ to $31.27\pm0.31$ over three
seeds ($10$K steps, learning rate $10^{-3}$): a $26.4$ PPL gap against per-cell seed spreads of $0.1$--$0.3$,
two orders of magnitude larger than the noise. The length caveat below is specific to canonical GaLore's small
margin; the LDAdam gap is present at both horizons, $26.4$ at its own $10$K recipe and $1.81$ at the common
$40$K recipe (Table~\ref{tab:transport}), so it is not a short-run artifact. We do not rank optimizers by gap
size: at the common recipe LDAdam's gap ($1.81$) is smaller than canonical GaLore's ($2.79$), and the $26.4$
belongs to its own recipe only. APOLLO-Mini moves from $97.83/101.77$ at
$\beta_2{=}0.999$ to $96.76/97.20$ at $\beta_2{=}0.99$ (two seeds); it resamples a random Gaussian projection
every $T$ rather than an SVD top-$r$ subspace, so two consecutive frames are near-orthogonal by construction
and the refresh geometry differs. We read APOLLO as a directional consistency check only, not a third strong
leg: two seeds, small margins, high absolute perplexity. The direction is nonetheless the same.

\paragraph{A length caveat, stated plainly.} For canonical GaLore the margin is training-length dependent. It
is present at $40$K (above) but absent, even slightly reversed, in shorter $10$--$20$K runs: the staleness cost
accrues over rotations, and the second-moment bias-correction transient occupies a larger fraction of a short
run. The effect is real but integrates with training; short-horizon tuning understates it.

\paragraph{The causal control.} The check that ties the $\beta_2$ effect to the refresh, and not to these
training runs in general, is full-rank AdamW: the same model, data, and schedule with no subspace refresh.
There the preference reverses to the standard $\beta_2{=}0.999$ at both learning rates tested ($97.9$ against
$118.0$ at learning rate $6\times10^{-4}$, and $158.3$ against $200.9$ at $10^{-3}$; $14.4$K steps, one seed),
exactly as the forgetting-factor argument predicts when the target is never reset. The control is directional
(one seed, two learning rates, shorter horizon) but the margins are far outside seed noise, and its absolute
perplexities are not comparable to the $40$K GaLore runs above. The lower-$\beta_2$ advantage is therefore
specific to subspace refreshing. A second moment that a refresh periodically invalidates is what makes $0.999$
the wrong default.

\begin{figure}[t]
\begin{center}
\includegraphics[width=0.56\linewidth]{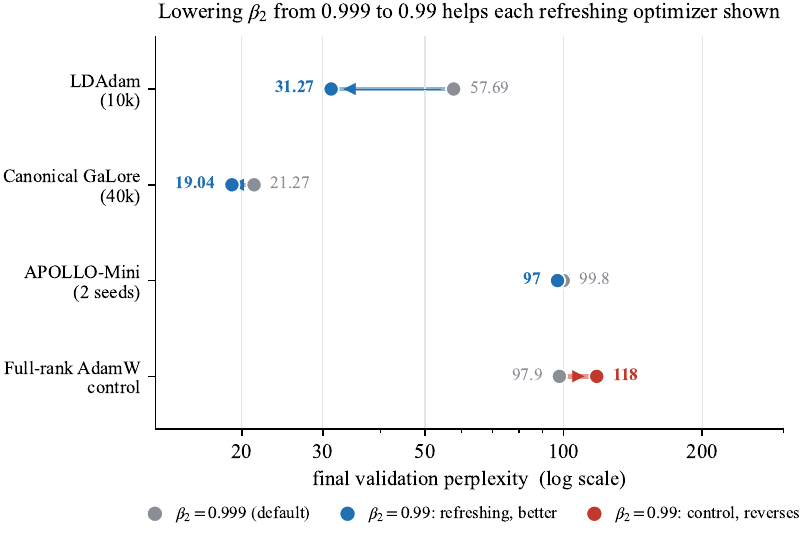}
\end{center}
\caption{Lowering $\beta_2$ from the $0.999$ default to $0.99$ at Pythia-1B (final validation perplexity, log
scale). The axis is logarithmic and each optimizer sits at its own training recipe (steps, learning rate, and
length differ across the three, and the full-rank control); the comparison is within-optimizer, where each $\beta_2$
pair shares its recipe, not across optimizers. For each refreshing optimizer shown, the lower value is better: LDAdam by $57.69\to31.27$ at its own $10$K recipe
(three seeds; at the common $40$K recipe its gap is $1.81$, Table~\ref{tab:transport}), canonical GaLore by
$21.27\to19.04$ (three seeds), APOLLO-Mini directionally (two seeds). Gap sizes are not comparable across
optimizers because recipes differ. The full-rank AdamW control, which runs with no subspace refresh, reverses
the preference back to $0.999$, directionally isolating the periodic redraw as the cause.
The canonical-GaLore advantage is flat across the EMA-basis averaging window ($\beta_2$ gap
$1.94$--$2.27$ PPL over $\beta_3\in\{0,\dots,0.9999\}$, \S\ref{sec:remedies:averaging}), so a shorter memory and
basis averaging act independently. Exact values, seed counts, and recipe: Appendix~\ref{app:configs},
Table~\ref{tab:beta2}.}
\label{fig:beta2}
\end{figure}

\subsection{Prescription}
\label{sec:remedies:prescription}

The geometry gives a recipe. Measure how many gradient directions your layer actually resolves, the small
determined spike of \S\ref{sec:tumble:gap} rather than the nominal rank; call it $k^\star$. If the projection
rank $r$ exceeds $k^\star$, the directions beyond $k^\star$ sit in the gapless tail and are re-drawn at each
refresh, so the state carried across the refresh goes stale. Then do two things. Transport the optimizer state through the rotation:
the first moment by $R$ (Proposition~\ref{prop:mmse}), the second moment by the symmetric-square rule of
Theorem~\ref{thm:blind_lower_bound}. And lower $\beta_2$ toward $0.99$. Both help, and their gains are partly
independent (Table~\ref{tab:transport}), so applying both is worth more than either alone. Do not expect basis
averaging to recover the difference: when the rank sits in the gapless tail there is no stable subspace to
average toward (\S\ref{sec:tumble:gap}), and four decades of averaging window leave both the perplexity and the
realized rotation unchanged (\S\ref{sec:remedies:averaging}).

% Related work — positioning, not a list. GaLore family; what "stability" measured; geometry; forgetting factors.
\section{Related work}
\label{sec:related}

\paragraph{The GaLore family and its premise.} GaLore~\citep{galore} and its refinements share one move: run an optimizer inside an explicitly recomputed
top-$r$ gradient subspace. The variations are in what they carry across the refresh: LDAdam transports both
moments and adds error feedback~\citep{ldadam}, SubTrack++ tracks the subspace incrementally~\citep{subtrackpp},
FRUGAL keeps a reduced state~\citep{frugal}, and GoLore~\citep{golore} and APOLLO~\citep{apollo} alter the
projection itself, while Q-GaLore quantizes it and adapts the refresh rate per layer~\citep{qgalore}. Memory-efficient
optimizers outside the projection family instead compress the second moment directly, by low-rank
factorization~\citep{adafactor} or block-wise learning-rate sharing~\citep{adammini}. The shared, usually unstated, premise is that
this subspace is a meaningful object that persists across refreshes. Our contribution is orthogonal to these
methods: we do not propose another member of the family but measure the premise they rest on, find it false at the
ranks they use, and show both why they nonetheless succeed (energy capture, \S\ref{sec:energy}) and what the
redraw costs them (second-moment staleness, \S\ref{sec:staleness}). \citet{ldadam} note from the theory side
that GaLore requires a ``strong stable-rank assumption'' but do not measure it; we measure it. LDAdam also supplies
the state-transport rule we build on: it transports both moments by the change of basis and adds error feedback. We
do not prove that its second-moment rule is optimal; we prove that first-moment transport is the MMSE map and that a
rotation-blind second-moment carry is a factor $\approx (r{-}k^\star)/2$ suboptimal (Theorem~\ref{thm:blind_lower_bound},
Appendix~\ref{app:weingarten}), a bound LDAdam escapes because it observes the rotation. That LDAdam already
defaults to $\beta_2{=}0.99$ is, on our account, not a tuning accident but a symptom of exactly the staleness we
characterize.

\paragraph{What prior ``stability'' measurements measured.} Reports that the GaLore subspace is stable rest on
quantities that remain small while the sign-corrected chordal rotation is near-maximal: a distribution of principal
angles~\citep{frugal}, which leaves the dominant tail contribution to chordal distance unmeasured, and incremental-tracker self-consistency
~\citep{subtrackpp}, a property of the tracker rather than of a fresh SVD (\S\ref{sec:tumble:wrongquantity}). None
measured $\DeltaU$, the quantity that determines what happens to carried state.

\paragraph{Random and averaged subspaces.} Two recent works reach compatible conclusions from the randomization
side. \citet{randsubspaces2510} argue that the gradient subspace landscape has nearly flat curvature and that much
of the gradient energy lies outside any core subspace, so random low-rank projections preserve the relevant
geometry and match SVD-based selection; they exploit this to navigate the Grassmannian with random steps. \citet{gum}
independently diagnose that noise-dominated singular directions make the SVD projection unreliable, though they
debias GaLore by probabilistically sampling full-rank layer updates rather than by randomizing the projection. Both
locate an unreliability in the SVD-selected frame past the spike that we trace to the same-step non-identifiability
of the top-$r$ cut. A separate line averages the subspace rather than randomizing it: \citet{alice2502} maintain an adaptively
estimated low-dimensional subspace instead of a fresh SVD at each refresh, closest in spirit to the averaged-basis
variant we test in \S\ref{sec:remedies:averaging}. Our EMA-basis measurement is our own implementation at
Pythia-1B with $r{=}128$ and speaks only to that setting; we make no claim about the performance reported for Alice.

\paragraph{Gradient geometry and subspace perturbation.} That optimization concentrates in a small gradient
subspace is established~\citep{gurari}, and gradient spectra serve directly as optimization
diagnostics~\citep{spectrallens}; we reconcile it (\S\ref{sec:tumble:gap}): the small determined subspace is
the spike, and $r$ overshoots it into the gapless tail. The link from a vanishing spectral gap to maximal subspace
sensitivity is the Davis--Kahan and Wedin $\sin\theta$ theory~\citep{davis_kahan,wedin}; the maximal-rotation value
$\sqrt{2r}$ is the diameter of the Stiefel manifold~\citep{grassmannhandbook,meckes2019compact}. The Weingarten
calculus on the orthogonal group~\citep{collins_matsumoto2009} enters separately, bounding the second-moment error
that the rotation induces (\S\ref{sec:staleness}).

\paragraph{Forgetting factors and tracking.} Treating $\beta_2$ as a forgetting factor for a second moment that
tracks an abruptly-reset target places our staleness result in the adaptive-filtering tradition: the bias--variance
behavior of the forgetting-factor RLS algorithm under parameter jumps is classical~\citep{guoljung}, and our
prediction $\beta_2^\star<0.999$ under a periodically re-drawn subspace is its specialization to GaLore's periodic
refresh. The
distinctive element here is that the ``jumps'' are not external nonstationarity but are \emph{imposed by the
optimizer itself} at every refresh. GaLore's periodic SVD is likewise the batch counterpart to incremental
subspace tracking (Oja's rule~\citep{oja1982}, GROUSE~\citep{grouse}, PAST~\citep{yang1995past}); the same-step
result bounds all of these, since a subspace two simultaneous batches cannot agree on is recoverable by no tracker.
Unlike convergence analyses of low-rank optimizers~\citep{golore,gum}, we characterize the identifiability of the
projected subspace itself.

% Limitations + Conclusion v2 (final push). Honest scoping; heavy-handed reflective close.
% No em dashes. Claims trace to cards; limitation list mirrors the reviewer-attack surface.
\section{Limitations}
\label{sec:limitations}

Five boundaries of the evidence are worth stating plainly. First, the full identifiability
apparatus (the four-layer same-step control and the $N^{-1/4}$ averaging-budget test) is developed at
Pythia-160M; the same-step disagreement and $k^\star$ are additionally measured on pretrained
checkpoints up to 6.9B (\S\ref{sec:noident:general}), and at 1B we further verify that basis averaging
neither stabilizes the subspace nor improves perplexity. The $N^{-1/4}$ averaging-budget sweep, developed at 160M, is confirmed at 410M, 1B, and 2.8B and
reproduced across GPU backends (Appendix~\ref{app:probes}); only the closed-form out-of-sample fit is
reported solely at 160M. Second, the
basis-averaging sweep is one seed per cell; we read it as a direction, supported by the measured
rotation and a second-cluster replication, not as a headline number. Third, the transport arm of
our $2{\times}2$ is the full LDAdam update, which combines state transport with error feedback;
the comparison bounds their joint effect and does not separate them. Fourth, perplexity results
are single-recipe per table (recipes stated in each caption; magnitudes shift between recipes,
directions do not), on WikiText and Pythia only. Fifth, the canonical-GaLore $\beta_2$ margin is
training-length dependent: present at 40K steps, absent at 10--20K, so short-horizon tuning will
not see it. Our claims are scoped to optimizers that explicitly compute and refresh a gradient
subspace; LoRA and other parameterizations make no such assumption, and our measurements say
nothing about them.

\section{Conclusion}
\label{sec:conclusion}

The premise beneath GaLore-family optimizers is that there is a low-rank gradient subspace to
track. We measured the premise. Beyond a reproducible core of roughly $39$ of $128$ directions,
the object is not there: two estimates of the subspace taken at the same instant disagree as much
as estimates taken a full refresh apart, the spectrum offers no gap for a rank cut to rest on,
and the tail that a larger sample might average away turns out to be signal, shrinking four times
too slowly for any budget to rescue. What looked like a subspace rotating is a subspace being
re-drawn.

Once that is said, the engineering sorts itself. Stabilizing the basis is chasing the phantom: at
1B, four decades of averaging window change neither the realized rotation nor the loss. The state,
not the frame, is what must move: the first moment transports exactly through the rotation, the
second through its squared entries, blind carry is provably a factor $(r{-}k^\star)/2$ worse, and the
transported optimizer (the LDAdam update, transport plus error feedback) beats the best
untransported configuration in a controlled three-seed comparison. The cheaper fix, shortening the second moment's memory to $\beta_2{=}0.99$, helps the
refreshing optimizers we tested at their reported recipes (small and recipe-sensitive for canonical GaLore) and
reverses in a full-rank control, because a memory of a
thousand steps is a liability precisely when the coordinates it averages over are re-drawn every
$160$.

The takeaway we want a reader to carry out of this paper is the probe, not the numbers. Before
building on a top-$r$ gradient subspace, measure how much of it is real: split a batch, compute
the subspace twice, and count the directions that agree. If your rank exceeds that count, the
excess is noise that will be re-drawn at every refresh, and the right response is not a better
tracker but transported state and a shorter memory. Low-rank training works, and it works for a
reason; the reason is just not the one its geometry seemed to promise.

\bibliography{bpdr}
\bibliographystyle{preprint_style}

\appendix
% Appendix — curated from proofs/{L1,L4,L3}. RETRACTED content stripped (no warm-v, no T_c, no uncarded 54.0 PPL).
% Resolves: app:deltau, app:deltau_scaling, app:weingarten (thm:blind_lower_bound), prop:mmse, app:coverage.

\section{The sign-corrected rotation metric and the \texorpdfstring{$\sqrt{2r}$}{sqrt(2r)} scaling}
\label{app:deltau}

\paragraph{Sign artifact in the raw distance.}
Singular vectors carry an arbitrary per-column sign: the $i$-th column of $U$ may be negated without changing the
subspace. Computed directly, $\|\Unew-\Uold\|_F$ is therefore dominated by sign inconsistencies rather than genuine
rotation. For two $m\!\times\!r$ orthonormal matrices with independent random column signs,
\[
  \EE\bigl[\|\Unew-\Uold\|_F^2\bigr] = 2r - 2\,\EE\bigl[\mathrm{tr}(\Unew^\top\Uold)\bigr] \approx 2r,
\]
since the $r$ diagonal entries of $R=\Unew^\top\Uold$ have independent random signs and mean zero. The raw metric
thus returns $\approx\sqrt{2r}$ regardless of $T$, a sign artifact rather than a measurement of rotation.

\paragraph{Corrected metric.}
The quantity we compute and report is the principal-angle (chordal) form
$\DeltaU=(2r-2\sum_i\cos\theta_i)^{1/2}$ with $\cos\theta_i=\sigma_i(\Unew^\top\Uold)$ the singular values of the
cross-Gram matrix (Eq.~\ref{eq:deltaU}), which is sign-invariant by construction and measures genuine rotation
rather than the arbitrary column sign. (A per-column sign alignment
$s_i=\mathrm{sign}(e_i^\top R\,e_i)$ followed by a Frobenius difference upper-bounds this quantity, with equality
when the columns coincide with the principal vectors; we do not rely on that route.) We use $\DeltaU$ for the
singular-value form throughout.

\paragraph{The $\sqrt{2r}$ scaling.}
\label{app:deltau_scaling}
\begin{table}[h]
\caption{Subspace rotation per refresh, Pythia-1B, $T{=}50$, averaged over refresh events (the data plotted in
Figure~\ref{fig:deltaU}). $\DeltaU$ reaches $96.7$--$99.6\%$ of its frame maximum $\sqrt{2r}$ at every
rank; the implied mean angle is $86$--$90^\circ$.}
\label{tab:deltaU}
\begin{center}
\begin{tabular}{ccccc}
\toprule
Rank $r$ & $\DeltaU$ (measured) & $\sqrt{2r}$ (max) & $\DeltaU/\sqrt{2r}$ & mean angle \\
\midrule
$32$  & $7.93\pm0.05$  & $8.00$  & $0.991$ & $89^\circ$ \\
$64$  & $11.27\pm0.01$ & $11.31$ & $0.996$ & $90^\circ$ \\
$128$ & $15.47\pm0.03$ & $16.00$ & $0.967$ & $86^\circ$ \\
$256$ & $22.10\pm0.01$ & $22.63$ & $0.977$ & $87^\circ$ \\
\bottomrule
\end{tabular}
\end{center}
\end{table}
Table~\ref{tab:deltaU} reports $\DeltaU$ measured at refresh events for Pythia-1B across ranks; the mean-angle
column is $\arccos(1-\DeltaU^2/2r)$, the angle implied by the mean cosine. The ratio
$\DeltaU(r{=}256)/\DeltaU(r{=}32)=22.10/7.93\approx2.79$ matches $\sqrt{256/32}=\sqrt{8}\approx2.83$ to $1.5\%$,
confirming $\DeltaU\propto\sqrt{r}$. All ranks reach $96.7$--$99.6\%$ of $\sqrt{2r}$. The $\sqrt{2r}$ ceiling is the
maximum Frobenius distance between two $m\!\times\!r$ orthonormal matrices, attained when all $r$ principal angles
equal $\pi/2$, the independent-Haar-sampling upper bound~\citep{meckes2019compact}. That correlated adjacent SVD
bases nonetheless operate near this ceiling at every refresh is the non-trivial empirical fact established in
\S\ref{sec:tumble}.

\section{Blind-estimator lower bound for the second moment (Weingarten)}
\label{app:weingarten}

This bound formalizes why carrying the second moment unchanged across a refresh is structurally suboptimal when
the estimator does not observe the rotation, using the Weingarten calculus for Haar-random orthogonal
matrices~\citep{collins_matsumoto2009}.

\begin{theorem}[Blind-estimator lower bound]
\label{thm:blind_lower_bound}
Let $v_{\mathrm{old}}=\lambda=(\lambda_1,\dots,\lambda_r)$ be the true per-coordinate second moments in the current
subspace $U$ ($\lambda_i=\EE[g_i^2]$ for $g=U^\top\nabla\mathcal L$). Let the new subspace be $U'=UR$ with
$R\sim\mathrm{Haar}(O(r))$, and let $\hat v$ be any estimator measurable with respect to $\sigma(\lambda)$ only
(i.e.\ blind to $R$). Then the optimal new-basis values are $v^*(U')_i=\sum_j\lambda_j R_{ji}^2$, and
\begin{equation}
  \EE_R\bigl[\|\hat v-v^*\|^2\bigr]=\|\hat v-\bar\lambda\mathbf 1\|^2+\frac{2r\,\mathrm{Var}(\lambda)}{r+2}
  \;\ge\;\frac{2r\,\mathrm{Var}(\lambda)}{r+2},
  \label{eq:blind_lb}
\end{equation}
with equality iff $\hat v=\bar\lambda\mathbf 1$ (the mean predictor). Canonical carry $\hat v=\lambda$ incurs
$\EE_R[\|\lambda-v^*\|^2]=r(r+4)\,\mathrm{Var}(\lambda)/(r+2)$, a factor $(r+4)/2\approx r/2$ above the floor.
\end{theorem}

\begin{proof}
By the Weingarten calculus~\citep{collins_matsumoto2009}, $\EE_R[v^*(U')_i]=\bar\lambda$ and
$\mathrm{Var}_R[v^*(U')_i]=2\,\mathrm{Var}(\lambda)/(r+2)$. Since $\hat v$ is independent of $R$, the
bias--variance decomposition gives $\EE_R[(\hat v_i-v^*_i)^2]=(\hat v_i-\bar\lambda)^2+2\,\mathrm{Var}(\lambda)/(r+2)$.
Summing over $i$ and minimizing over $\hat v$ yields~\eqref{eq:blind_lb}; substituting $\hat v=\lambda$ gives the
canonical-carry cost.
\end{proof}

\begin{remark}[Representation-theoretic reading]
The first moment transforms in the defining representation of $O(r)$ ($\mathbf m\mapsto R\mathbf m$), whereas the
optimal second moment transforms via $M(R)_{ij}=R_{ji}^2$, which fixes $\mathbf 1$ and decomposes into a trivial
component (the mean $\bar\lambda$) plus a sum-zero part. A blind estimator sees only the rotation-invariant
$\lambda$, so any $O(r)$-equivariant linear prediction lies in the trivial component: the unique equivariant blind
predictor is $\bar\lambda\mathbf 1$. By Schur's lemma the first moment admits an equivariant transport
($R\mathbf m$) but the second moment does not. This is the structural reason canonical carry is $(r/2)$-suboptimal
rather than merely worse in practice. Methods that \emph{observe} $R$~\citep{ldadam,subtrackpp} fall outside the
blind class and are not bound by~\eqref{eq:blind_lb}. The bound takes $R\sim\mathrm{Haar}(O(r))$, the
fully-degenerate limit; with a reproducible core of $k^\star$ directions (\S\ref{sec:noident:kstar}) the
Haar-resampled block has dimension $r-k^\star$ and the effective suboptimality factor scales as
$(r-k^\star)/2$, still ${\sim}45\times$ at the measured $k^\star\!\approx\!39$, $r{=}128$.
\end{remark}

\paragraph{Worked moments.}
The two moments the proof invokes follow from the second- and fourth-order statistics of a single Haar-random
frame. A column $R_{\cdot i}$ of $R\sim\mathrm{Haar}(O(r))$ is a uniform point on the unit sphere $S^{r-1}$, so its
entries $x_j:=R_{ji}$ satisfy $x_j^2\sim\mathrm{Beta}\!\bigl(\tfrac12,\tfrac{r-1}{2}\bigr)$ and, for $j\neq k$,
\begin{equation}
  \EE[x_j^2]=\frac1r,\qquad \EE[x_j^4]=\frac{3}{r(r+2)},\qquad \EE[x_j^2 x_k^2]=\frac{1}{r(r+2)}.
  \label{eq:haar_moments}
\end{equation}
These are consistent with $\sum_j x_j^2=1$: summing gives $\sum_j\EE[x_j^4]+\sum_{j\neq k}\EE[x_j^2 x_k^2]
=\tfrac{3}{r+2}+\tfrac{r-1}{r+2}=1$. Writing $v^*(U')_i=\sum_j\lambda_j R_{ji}^2$ and $S_1=\sum_j\lambda_j$,
$S_2=\sum_j\lambda_j^2$, the first moment is $\EE_R[v^*(U')_i]=S_1/r=\bar\lambda$, and
\begin{align}
  \EE_R\bigl[v^*(U')_i^2\bigr]
   &=\sum_j\lambda_j^2\,\EE[R_{ji}^4]+\sum_{j\neq k}\lambda_j\lambda_k\,\EE[R_{ji}^2 R_{ki}^2]
    =\frac{3S_2+(S_1^2-S_2)}{r(r+2)}=\frac{2S_2+S_1^2}{r(r+2)},\notag\\
  \mathrm{Var}_R\bigl[v^*(U')_i\bigr]
   &=\frac{2S_2+S_1^2}{r(r+2)}-\frac{S_1^2}{r^2}
    =\frac{2}{r(r+2)}\Bigl(S_2-\frac{S_1^2}{r}\Bigr)
    =\frac{2\,\mathrm{Var}(\lambda)}{r+2},
  \label{eq:vstar_var}
\end{align}
using $r\,\mathrm{Var}(\lambda)=S_2-S_1^2/r$. The canonical carry $\hat v=\lambda$ then costs
$\EE_R[\|\lambda-v^*\|^2]=\sum_i\bigl[(\lambda_i-\bar\lambda)^2+\mathrm{Var}_R(v^*_i)\bigr]
=r\,\mathrm{Var}(\lambda)+\tfrac{2r\,\mathrm{Var}(\lambda)}{r+2}=\tfrac{r(r+4)}{r+2}\mathrm{Var}(\lambda)$,
whereas the mean predictor $\hat v=\bar\lambda\mathbf 1$ attains the floor $\tfrac{2r}{r+2}\mathrm{Var}(\lambda)$;
their ratio is $(r+4)/2$, the factor quoted in Theorem~\ref{thm:blind_lower_bound}.

\section{First-moment transport is the MMSE linear map}
\label{app:mmse}

\begin{proposition}[MMSE first-moment transport]
\label{prop:mmse}
Let $\mathbf m=\Uold^\top g$ with $g\sim\mathcal N(0,\sigma^2 I)$, and consider estimating
$\hat m=\Unew^\top g$ from $\mathbf m$. Among linear estimators $A\mathbf m$, the minimum-mean-squared-error
choice is $A=\Unew^\top\Uold=R$, i.e.\ $\hat m=R\,\mathbf m$.
\end{proposition}

\begin{proof}
Write $g=\Uold\mathbf m+(I-\Uold\Uold^\top)g$, so $\hat m=R\mathbf m+\Unew^\top(I-\Uold\Uold^\top)g$. The second
term has zero mean and is uncorrelated with $\mathbf m$ under isotropy, so
$\EE\|\hat m-A\mathbf m\|^2=\EE\|R\mathbf m-A\mathbf m\|^2+\sigma^2\|\Unew^\top(I-\Uold\Uold^\top)\|_F^2$. The second
term is irreducible; the first is minimized uniquely at $A=R$.
\end{proof}

\noindent The isotropy assumption is exact only in the near-degenerate tail of the gradient spectrum, which is
precisely the part of the subspace that redraws (\S\ref{sec:noident:gap}); the determined spike is anisotropic but
stable and does not rotate, so $R$ is the optimal transport over exactly the directions that move.

\section{Coverage equals cumulative spectral mass}
\label{app:coverage}

\begin{proposition}[Coverage growth]
\label{prop:coverage}
If the gradient second-moment $\Sigma_g$ has eigenvalues $\lambda_1\ge\cdots\ge\lambda_n$ and $U$ spans the top-$r$
eigenvectors, then $\EE_g[\gamg]=\sum_{i=1}^r\lambda_i/\sum_{i=1}^n\lambda_i$.
\end{proposition}

\begin{proof}
$\EE[\|UU^\top g\|^2]=\mathrm{tr}(UU^\top\Sigma_g UU^\top)=\sum_{i=1}^r\lambda_i$ and
$\EE[\|g\|^2]=\mathrm{tr}(\Sigma_g)=\sum_{i=1}^n\lambda_i$.
\end{proof}

\noindent When a few eigenvalues dominate (the spike of \S\ref{sec:tumble:gap}), $\EE[\gamg]$ rises steeply for
$r$ up to the spike width and then flattens: past the spike, adding tail directions of nearly equal energy raises
coverage only marginally and almost independently of \emph{which} tail directions are chosen, the property that
makes a re-drawn frame an essentially equivalent projector (\S\ref{sec:energy}).

\section{Experimental configurations and reproducibility}
\label{app:configs}

\paragraph{Models and corpora.} Spectrum and drift measurements use Pythia-160M; the $\beta_2$ validation and the
$\DeltaU$ rotation law use Pythia-1B. Language-model training is on WikiText; the drift-rate invariance check
additionally uses C4 (below). Runs were executed on two independent clusters, NVIDIA L40S and AMD MI210, which we
use to cross-check the staleness measurements.

\paragraph{$\beta_2$ runs (Figure~\ref{fig:beta2}).} Table~\ref{tab:beta2} gives the exact values; each row has
its own recipe, so perplexities are comparable within rows only (LDAdam's $57.69$ here and its $18.73$ in
Table~\ref{tab:transport} are the same optimizer at $10$K/lr $10^{-3}$ and $40$K/lr $4\times10^{-3}$
respectively). Canonical GaLore
uses the configuration $r{=}128$, refresh $T{=}160$, learning rate $4\times10^{-3}$, projection scale $0.25$,
cosine schedule with warmup $4000$, $40$K steps, three seeds per $\beta_2$; LDAdam uses the faithful official implementation at $r{=}128$, learning rate
$10^{-3}$, $10$K steps, three seeds per $\beta_2$; APOLLO-Mini uses two seeds. The full-rank AdamW control runs the
identical data and schedule with no projection.

\begin{table}[h]
\caption{Lowering $\beta_2$ from the $0.999$ default to $0.99$ at Pythia-1B (the data of Figure~\ref{fig:beta2}).
For each refreshing optimizer shown, the lower value is better; for full-rank AdamW the preference reverses. LDAdam's gap
is largest at its own $10$K recipe, but gap sizes are not comparable across recipes and we do not read it as a
larger staleness effect (\S\ref{sec:staleness:val}). APOLLO-Mini (below the second rule) refreshes a random
Gaussian projection rather than an SVD top-$r$ subspace and enters as a directional check only. The
canonical-GaLore margin is training-length dependent (\S\ref{sec:staleness:val}).}
\label{tab:beta2}
\begin{center}
\small
\begin{tabular}{lccccc}
\toprule
Optimizer & rot.\ & $\beta_2{=}0.999$ & $\beta_2{=}0.99$ & $\Delta$ (PPL) & steps/seeds \\
\midrule
LDAdam~\citep{ldadam}           & yes & $57.69\pm0.11$ & $31.27\pm0.31$ & $\mathbf{26.4}$ & 10k/3 \\
Canonical GaLore~\citep{galore} & yes & $21.27\pm0.17$ & $19.04\pm0.13$ & $2.23$ & 40k/3 \\
\midrule
APOLLO-Mini~\citep{apollo}      & yes & $99.8$         & $97.0$         & $2.8$ & --/2 \\
\midrule
Full-rank AdamW (control)       & no  & $97.9$ & $118.0$ & $-20.1$ (reversed) & 14.4k/1 \\
\bottomrule
\end{tabular}
\end{center}
\end{table}

\paragraph{Spectrum probe (Figure~\ref{fig:spectrum}).} At Pythia-160M and 1B we take the per-matrix gradient $G$
for the embedding, attention query/key/value, attention output, and MLP weight matrices, compute its singular
values, and report: the adjacent-ratio $\sigma_r/\sigma_{r+1}$ at $r{=}128$; the maximum adjacent ratio anywhere in
the bulk ($\max_i \sigma_i/\sigma_{i+1}$ over the smooth region, ``max bulk gap''); the effective rank $\exp(H)$
with $H$ the Shannon entropy of the normalized squared spectrum; and the top-$r$ nuclear fraction
$\sum_{i\le r}\sigma_i/\sum_i\sigma_i$. The two scales agree closely: mean $\sigma_r/\sigma_{r+1}$ at $r{=}128$
is $\approx 1.005$ at both scales ($1.001$--$1.010$ at $160$M, $1.001$--$1.017$ at $1$B); max bulk gap stays in $1.011$--$1.022$ ($160$M) and $1.011$--$1.030$
($1$B); effective rank spans $89$--$280$ ($160$M) and $97$--$358$ ($1$B), straddling $r$ at both; and the top-$r$
nuclear fraction is $0.52$--$0.68$ ($160$M) and $0.46$--$0.65$ ($1$B). Measurements are taken over four layer types
at steps $80/100/120$.

\paragraph{Rotation probe ($\DeltaU$, $\alpha_{\mathrm{drift}}$).} Both quantities use the same primitive, the mean
principal angle between two orthonormal bases, $\overline{\theta}=\arccos\bigl(\mathrm{svdvals}(U_{\rm new}^\top
U_{\rm old})\bigr)$ averaged over columns. $\DeltaU$ (Eq.~\ref{eq:deltaU}) applies it across a refresh of the
imposed GaLore subspace; $\alpha_{\mathrm{drift}}$ (Appendix~\ref{app:drift}) applies it per step to the intrinsic
top-$r$ gradient-covariance subspace.

\section{Identifiability probe protocols}
\label{app:probes}

The same-step control (\S\ref{sec:noident:samestep}), the reproducible count $k^\star$
(\S\ref{sec:noident:kstar}), and the averaging-budget scaling (\S\ref{sec:noident:budget}) all read from one
construction, described here so they reproduce without the code. All three run on a pretrained Pythia-160M in
\texttt{float32}, evaluation mode (no dropout; gradients still flow), $r{=}128$, on the four layer types
\{\texttt{embed\_in}, layer-0 attention \texttt{query\_key\_value}, attention \texttt{dense}, MLP
\texttt{dense\_h\_to\_4h}\}.

\paragraph{Gradient pool.} We draw $K{=}128$ minibatches, each $B{=}8$ sequences of length $512$ taken from
disjoint, non-overlapping slices of WikiText-103, and for each compute the per-matrix gradient of the
next-token cross-entropy loss. For a gradient $G$ we work on the smaller side, $M=G$ if $\mathrm{rows}\le\mathrm{cols}$
else $G^\top$, and define its top-$r$ subspace as the top-$r$ eigenvectors of $MM^\top$ (equivalently the top-$r$
left singular vectors). Because every minibatch is evaluated at the \emph{same} checkpoint, two subspaces built
from two pools differ only in the sample, never in elapsed training.

\paragraph{Same-step disagreement and $k^\star$ (Figure~\ref{fig:samestep}a).} For an averaging fold
$N\in\{1,2,4,8,16,32,64\}$ and $4$ repetitions, we draw $2N$ distinct pool indices, split them into two halves of
$N$, average each half into $\bar G_A,\bar G_B$, and take their top-$r$ subspaces $U_1,U_2$. We report
$\DeltaU=(2r-2\sum_i\sigma_i(U_1^\top U_2))^{1/2}$ (Eq.~\ref{eq:deltaU}), the subspace overlap
$\mathrm{ov}=\sum_i\sigma_i(U_1^\top U_2)^2/r\in[0,1]$, and $k^\star(N)=r\cdot\mathrm{ov}$, the effective number of
shared directions. At $N{=}1$ this is the same-step split-batch disagreement ($0.725\,\sqrt{2r}$ over four layers; a four-layer-mean
$k^\star\!\approx\!39$, $48$ on the query-key-value matrix, per-layer standard error below $1$ over the $12$
repetitions of the trajectory run below). The across-time reference $0.742\,\sqrt{2r}$ is a separate
measurement: the sign-corrected rotation $\DeltaU$ across one refresh interval of $160$M training, not on the frozen
pool. The frozen same-step value is the noise floor; the across-time rotation exceeds it only marginally.

\paragraph{The floor is not specific to the pretrained checkpoint.} A natural objection is that the same-step
floor is read from a pretrained checkpoint while the across-time rotation is read during training, so the two
sit at different optimization states. We therefore reran the $N{=}1$ probe unchanged (same four layers,
$K{=}128$ pool, $r{=}128$, $12$ repetitions) at five points along the released Pythia-160M training trajectory
(revisions step$1000$, step$10000$, step$40000$, step$80000$, step$143000$). The four-layer same-step
disagreement is $\{0.648,\,0.668,\,0.700,\,0.708,\,0.710\}\,\sqrt{2r}$: lowest early in training and rising
toward the pretrained value, but never below $0.64\,\sqrt{2r}$. The reproducible core is correspondingly larger
early (query-key-value $k^\star$ $50.9,58.5,52.8,51.4,51.6$ across the same checkpoints). At every training state
the across-time rotation ($0.742$) therefore exceeds the same-step floor by at most a few percent of $\sqrt{2r}$,
so the rotation is dominated by non-identifiability rather than by motion of a stable subspace, and the
conclusion does not depend on the state at which the floor is measured.

\paragraph{$N$-averaging spectra and the budget test (Figure~\ref{fig:samestep}b).} The same run additionally
logs the singular spectrum of each averaged matrix $\bar G_A$ (densely over the top $300$ ranks, log-spaced in the
tail). Three analyses follow. (i)~\emph{Deep-tail scaling}: a tail singular value at fixed rank, scaled by
$\sqrt N$, is flat if the tail is pure sampling noise ($\sigma\propto N^{-1/2}$); the measured value instead rises
$2.0$--$3.2\times$ from $N{=}1$ to $64$, i.e.\ $\sigma\propto N^{-1/4}$. (ii)~\emph{Out-of-sample prediction}: a
spiked-covariance (BBP) model~\citep{bbp2005} with a single noise scale $\nu_1$ per layer, calibrated only to reproduce the
measured overlap at $N{=}1$, predicts $\mathrm{ov}(4N)$ from the de-censored spectrum of $\bar G_N$ (setting the
noise floor to $\nu_1/\sqrt N$); over $20$ held-out points the mean absolute error is $0.073$, above the
pre-registered pass threshold of $0.03$, which is why we report a measured scaling and not a closed-form law.
(iii)~\emph{De-censored exponent}: the log--log slope of ranks $3$--$60$ of the $N{=}64$ spectrum gives a signal
power law with exponent $\alpha=1.03$--$1.35$, agreeing with the exponent implied by the overlap-versus-$N$ curve.

\paragraph{The averaging-budget result holds across scale and GPU backend.} We reran the same-step and
$N$-averaging probe unchanged on pretrained checkpoints at 410M, 1B, and 2.8B. At every scale, averaging to
$N{=}64$ ($64\times$ the gradient sample) removes only $24$--$29\%$ of the single-batch same-step disagreement
(Table~\ref{tab:budget_scale}), the same weak decay measured at 160M, so no averaging budget restores
identifiability at any scale tested. As a numerics control, rerunning the 160M probe on a second GPU backend
(AMD MI210, ROCm) reproduces the reference measurement (NVIDIA L40S, CUDA) to within $0.3\%$ ($0.709$ versus
$0.707$ at $N{=}1$), confirming the persistent tail is a property of the gradient, not of the linear-algebra
backend.

\begin{table}[h]
\caption{Same-step disagreement decay under $N$-fold averaging, across scale (pretrained checkpoints, $r{=}128$,
four-layer mean). Entries are the $N{=}64$ disagreement as a fraction of the $N{=}1$ disagreement: $1.0$ would mean
averaging changed nothing and $0$ would mean it fully resolved the subspace. Averaging removes only about a quarter
at every scale; the reproducible core $k^\star$ separately shrinks with scale (Table~\ref{tab:generality}).}
\label{tab:budget_scale}
\begin{center}
\begin{tabular}{lcccc}
\toprule
Scale & 160M & 410M & 1B & 2.8B \\
\midrule
$\DeltaU/\sqrt{2r}$ at $N{=}64$, relative to $N{=}1$ & $0.76$ & $0.74$ & $0.72$ & $0.71$ \\
\bottomrule
\end{tabular}
\end{center}
\end{table}

\section{Generality of same-step non-identifiability}
\label{app:generality}

Table~\ref{tab:generality} gives the per-model same-step disagreement (\S\ref{sec:noident:general},
Figure~\ref{fig:generality}). The probe is the same-step protocol of Appendix~\ref{app:probes} run unchanged on
pretrained checkpoints (all trained with full-rank Adam), $r{=}128$, $K{=}32$ gradient matrices of $B{=}4$
sequences, four layer types (embedding, attention input, attention output, MLP), four repetitions; the language
models use WikiText, the vision transformer uses CIFAR-10 images. We report the mean over layers of the $N{=}1$
disagreement $\DeltaU/\sqrt{2r}$, the mean adjacent singular-value ratio at $r$, and the mean reproducible count
$k^\star$. The disagreement rises with scale and $k^\star$ falls; no model shows a spectral gap at $r$.

\begin{table}[h]
\caption{Same-step subspace disagreement across architectures, scales, and domains (pretrained, full-rank-Adam
checkpoints; $r{=}128$; mean over four layer types). The Pythia-160M row reproduces the main-text control
($0.73$ here at $B{=}4,K{=}32$ versus $0.725$ at $B{=}8,K{=}128$; $k^\star$ shifts by one under the same change of
settings, $38$ here versus the four-layer mean $39$ of the main text).}
\label{tab:generality}
\begin{center}
\small
\begin{tabular}{llccc}
\toprule
Model & Architecture & $\DeltaU/\sqrt{2r}$ ($N{=}1$) & adj.\ ratio at $r$ & $k^\star$ \\
\midrule
Pythia-70M   & GPT-NeoX & $0.67$ & $1.006$ & $50$ \\
Pythia-160M  & GPT-NeoX & $0.73$ & $1.006$ & $38$ \\
Pythia-410M  & GPT-NeoX & $0.76$ & $1.006$ & $33$ \\
Pythia-1B    & GPT-NeoX & $0.81$ & $1.006$ & $25$ \\
Pythia-1.4B  & GPT-NeoX & $0.80$ & $1.005$ & $25$ \\
Pythia-2.8B  & GPT-NeoX & $0.80$ & $1.007$ & $25$ \\
Pythia-6.9B  & GPT-NeoX & $0.83$ & $1.004$ & $21$ \\
\midrule
GPT-2-large  & GPT-2 (learned-pos, GELU) & $0.71$ & $1.007$ & $43$ \\
Qwen2.5-3B   & Llama-class (RoPE/SwiGLU) & $0.78$ & $1.007$ & $29$ \\
Llama-3-8B   & Llama-3 & $0.81$ & $1.009$ & $25$ \\
\midrule
ViT-B (CIFAR-10) & vision transformer & $0.53$ & $1.022$ & $76$ \\
\bottomrule
\end{tabular}
\end{center}
\end{table}

\paragraph{Rank versus $k^\star$: transport, do not truncate.} A natural reading of $k^\star$ is that the rank
should be reduced toward it. It should not. In a single-seed sweep at Pythia-1B ($\beta_2{=}0.99$, $T{=}160$,
$14$K steps), canonical GaLore reaches $27.9$ perplexity at $r{=}128$ against $29.8$ at $r{=}64$, and worse below
(this recipe is not comparable to the $40$K runs elsewhere; we compare only within it). The result is directional
but consistent with the coverage argument (\S\ref{sec:energy}): the directions beyond $k^\star$ are individually
non-identifiable yet still capture useful gradient energy, so the excess rank is worth keeping. The fix for the
stale state it carries is to transport it (\S\ref{sec:remedies:transport}), not to truncate the rank.

\section{The second moment lags on the same curve as full-rank Adam}
\label{app:overlay}

The mechanism of \S\ref{sec:staleness} predicts that the second moment's staleness should be governed by its memory
$\tau_v=1/(1-\beta_2)$, and that a GaLore-family optimizer at a given $\beta_2$ should be no more anomalous than
full-rank Adam at the same $\beta_2$. We confirm both. We define the per-cycle \emph{staleness lag} as the mean
distance between the carried second moment and its post-rotation optimum over a refresh cycle. Lag rises
monotonically with $\tau_v$: for full-rank Adam at Pythia-160M (AMD MI210, $4$K steps, two seeds), lag
$=0.11,0.14,0.19,0.25,0.26,0.29$ at $\beta_2=0.9,0.95,0.97,0.99,0.995,0.999$
($\tau_v=10,20,33,100,200,1000$), with seeds agreeing to ${\sim}0.01$. On a second cluster (NVIDIA L40S), GaLore
at matched $\beta_2$ lies on the same curve as full-rank Adam measured there: $0.13$ against $0.16$ at
$\beta_2{=}0.95$, and $0.30$ against $0.36$ at $\beta_2{=}0.99$ (within-cluster comparisons only). The lag is also flat across the refresh interval ($0.27$--$0.31$ over an $8\times$ range of
$T$), as expected for a quantity set by $\tau_v$ rather than by $T$. This replicates across both clusters (NVIDIA
L40S and AMD MI210, different framework versions and data pipelines).

We restrict this overlay claim to $\beta_2\le0.99$. At $\beta_2=0.999$ the instantaneous lag shows a large
transient that does not persist to steady state; we therefore do not read it as a steady-state divergence, and the
loss consequence of high-$\beta_2$ staleness is the integrated effect documented in \S\ref{sec:staleness:val}, not
an instantaneous lag explosion.

\section{A forgetting-factor bias--variance argument for \texorpdfstring{$\beta_2$}{beta2}}
\label{app:forgetting}

This appendix supports the qualitative claim of \S\ref{sec:staleness:val} that a subspace re-drawn every
$T\ll\tau_v$ steps pushes the error-minimizing second-moment decay below the no-reset optimum. It is a
specialization of the classical forgetting-factor bias--variance tradeoff~\citep{guoljung} to a target that the
optimizer resets itself. We do not derive an optimal $\beta_2$; we show the minimizer is interior and moves down
as the reset becomes more frequent, which is all the prescription rests on.

\paragraph{Model.} Fix one projected coordinate. The second moment follows the exponential moving average
$v_t=\beta_2 v_{t-1}+(1-\beta_2)x_t$, where $x_t=\tilde g_t^{\,2}$ is the instantaneous squared projected gradient,
an unbiased noisy reading of the true per-coordinate second moment with $\EE[x_t]=\lambda_t^\ast$ and
$\mathrm{Var}(x_t)=s^2$. The target $\lambda_t^\ast$ is piecewise constant, holding for the $T$ steps of a refresh
interval and jumping at each refresh when the basis is re-drawn. Under the Haar model of
Appendix~\ref{app:weingarten} the coordinate's target after a refresh has the same mean $\bar\lambda$ and a
mean-squared change $\overline{\Delta^2}:=\EE[(\lambda^\ast_{\mathrm{new}}-\lambda^\ast_{\mathrm{old}})^2]
=\Theta(\mathrm{Var}(\lambda))$ across the jump.

\paragraph{The two error terms.} Averaged over a refresh interval, the steady-state per-coordinate mean-squared
error $\EE[(v_t-\lambda_t^\ast)^2]$ splits into a tracking-noise term and a lag term:
\begin{equation}
  \mathrm{MSE}(\beta_2)\;=\;\underbrace{\frac{1-\beta_2}{1+\beta_2}\,s^2}_{\text{noise}}
  \;+\;\underbrace{\overline{\Delta^2}\,\frac{1-\beta_2^{2T}}{T\,(1-\beta_2^2)}}_{\text{lag}}.
  \label{eq:ff_mse}
\end{equation}
The noise term is the stationary variance of an EMA of i.i.d.\ readings, $\mathrm{Var}(v_\infty)=(1-\beta_2)s^2/(1+\beta_2)$;
it \emph{decreases} in $\beta_2$, so a longer memory averages away more sampling noise. The lag term is the
mean-squared residual bias after a target jump: right after a refresh the bias decays as $\beta_2^{k}$ over the $k$
steps since the jump, and averaging $\overline{\Delta^2}\beta_2^{2k}$ over $k=0,\dots,T-1$ gives
$\overline{\Delta^2}(1-\beta_2^{2T})/\bigl(T(1-\beta_2^2)\bigr)$; it \emph{increases} in $\beta_2$.

\paragraph{Why the optimum moves down under refresh.} Write $\tau_v=1/(1-\beta_2)$. For $\beta_2$ near $1$,
$1-\beta_2^2\approx 2/\tau_v$ and $\beta_2^{2T}\approx e^{-2T/\tau_v}$, so the lag term is
$\tfrac{1}{2}\overline{\Delta^2}\,(\tau_v/T)\bigl(1-e^{-2T/\tau_v}\bigr)$, the same $\tau_v(1-e^{-2T/\tau_v})/T$
shape that governs any EMA tracking a reset target. Two limits fix its behavior. When resets are rare relative to
the memory, $T\gg\tau_v$, the lag term is $O(\tau_v/T)\to 0$ and $\mathrm{MSE}$ is monotone decreasing in $\tau_v$:
the minimizer is $\beta_2\to1$, the standard no-reset preference. When resets are frequent, $T\ll\tau_v$,
$1-e^{-2T/\tau_v}\approx 2T/\tau_v$ and the lag term \emph{saturates} at $\overline{\Delta^2}$, its ceiling, once
$\tau_v\gtrsim T$. Past that point raising $\beta_2$ buys only the residual $O(s^2/\tau_v)$ of noise reduction
while holding near-maximal lag, so the minimizer $\beta_2^\star$ is interior and decreases as $T$ shrinks. This is
the direction \S\ref{sec:staleness:val} confirms empirically: the second moment, whose target GaLore re-draws
every $T\ll\tau_v$ steps, prefers $\beta_2{=}0.99$ over $0.999$, while the full-rank control, whose target is never
reset ($T=\infty$), keeps the $0.999$ preference. We do not claim $0.99$ is the derived optimum; the argument fixes
only the sign of the shift.

\section{The eigenbasis drift rate is architecture-intrinsic}
\label{app:drift}

The rotation law $\DeltaU\approx\sqrt{2r}$ concerns the \emph{imposed} GaLore subspace. For completeness we also
measure the drift rate of the model's \emph{intrinsic} gradient geometry: the per-step mean principal-angle
rotation of the top-$r$ gradient-covariance eigenbasis, $\alpha_{\mathrm{drift}}$. At Pythia-160M this is
$\alpha_{\mathrm{drift}}\approx0.05$ rad/step, and it is strikingly invariant: $0.049$--$0.054$ across all $\beta_2$,
$0.057$ (random sampling) versus $0.054$ (sequential), and $0.0558$ (WikiText) versus $0.0556$ (C4), a difference
below $0.5\%$. The drift rate is thus a property of the architecture, not of the optimizer setting, the sampling
order, or the corpus.

\begin{remark}[Scope of the ``staleness'' axis]
Because $\alpha_{\mathrm{drift}}$ is architecture-set and we have not varied it across architectures, we do not
claim a fully bivariate universal staleness law in $(\alpha_{\mathrm{drift}},\tau_v)$. The axis we validate is the
second-moment memory $\tau_v=1/(1-\beta_2)$ (\S\ref{sec:staleness}, Appendix~\ref{app:overlay}); the
$\alpha_{\mathrm{drift}}$ axis is reported as a measured constant for the architectures studied. We also note that
the imposed GaLore rotation arrives in bursts of size $\sqrt{2r}$ every $T$ steps, whereas intrinsic drift is
continuous, so the two enter a shared staleness picture only up to this burst-versus-continuous distinction.
\end{remark}

\end{document}